\documentclass[12pt]{article}
\usepackage{dsfont,amssymb,amsmath,amsthm}
\usepackage{graphicx, color}

\usepackage{geometry}
\usepackage{setspace}
\usepackage{enumerate}

\newtheorem{theorem}{Theorem}[section]
\newtheorem{lemma}[theorem]{Lemma}
\newtheorem{claim}[theorem]{Claim}

\newtheorem*{remark}{Remark}

\newtheorem*{nonumtheorem}{Theorem}

\usepackage[pdftex,pagebackref=true,colorlinks]{hyperref}

\newcommand{\R}{\mathbb R}
\newcommand{\dist}{\mathsf{dist}}

\newcommand{\Ex}{\mathop \mathbb{E}}
\newcommand{\Var}{\text{Var}}

\newcommand{\eps}{\epsilon}

\newcommand{\VC}{\text{VC}}

\begin{document}

\title{Teaching and compressing for low VC-dimension\footnote{A preliminary version of this work, combined with the paper ``Sample compression schemes for VC classes'' by the first and the last authors, was published in the proceeding of FOCS'15.}}

\author{Shay Moran\thanks{Departments of Computer Science, Technion-IIT, Israel and Max Planck Institute for Informatics, Saarbr\"{u}cken, Germany. {\tt  shaymrn@cs.technion.ac.il.}}
\and Amir Shpilka\footnote{Department of Computer Science, Tel Aviv University, Israel.
{\tt shpilka@post.tau.ac.il.}
The research leading to these results has received funding
from the European Community's Seventh Framework Programme (FP7/2007-2013) under grant agreement number 257575, and from the Israel Science Foundation (grant number 339/10).}
\and
Avi Wigderson\thanks{School of Mathematics, Institute for Advanced Study, Princeton NJ.
{\tt avi@ias.edu.}}
\and Amir Yehudayoff\thanks{Department of Mathematics, Technion-IIT, Israel.
{\tt amir.yehudayoff@gmail.com.}
Horev fellow -- supported by the Taub foundation.  Research is also supported by ISF and BSF.}}

\date{}

\maketitle

\begin{abstract}
In this work we study the quantitative relation between VC-dimension and two other basic parameters related to learning and teaching.
Namely, the quality of sample compression schemes and of teaching sets for classes of low VC-dimension.
Let $C$ be a binary concept class of size $m$ and VC-dimension $d$.  
Prior to this work,  the best known upper bounds for both parameters were 
$\log(m)$,
while the best lower bounds are linear in $d$. 
We present significantly better upper bounds on both as follows.
Set $k = O(d 2^d \log \log |C|)$.

We show that
there always exists a concept $c$ in $C$ with a teaching set 
(i.e. a list of $c$-labeled examples uniquely identifying $c$ in $C$) 
of size $k$.  
This problem was studied by Kuhlmann (1999).
Our construction implies that the recursive teaching (RT) dimension of $C$ is at most $k$ as well.  
The RT-dimension was suggested by
Zilles et al. and Doliwa et al. (2010).
The same notion (under the name partial-ID width) was independently studied by Wigderson and Yehudayoff (2013). 
An upper bound on this parameter that depends only on $d$  
is known just for the very simple case $d=1$, and is open even for $d=2$. 
We also make small progress towards this seemingly modest goal.

We further construct sample compression schemes
of size $k$ for $C$, with additional information of $k \log(k)$ bits. 
Roughly speaking, given any 
list of $C$-labelled examples of arbitrary length, 
we can retain only $k$ labeled examples 
in a way that allows to recover the labels of all others examples in the list,
using additional $k\log (k)$ information bits.
This problem was first suggested
by Littlestone and Warmuth (1986).
\end{abstract}

\newpage

\section{Introduction}

The study of mathematical foundations of learning and teaching
has been very fruitful,  revealing fundamental connections to
various other areas of mathematics, such as geometry, topology, and combinatorics.
Many key ideas and notions emerged from this study:
Vapnik and Chervonenkis's VC-dimension \cite{zbMATH03391742},
Valiant's seminal definition of PAC learning \cite{zbMATH03943062}, 
Littlestone and Warmuth's sample compression schemes \cite{littleWarm}, Goldman and Kearns's  teaching dimension~\cite{GoldmanK95}, recursive teaching dimension  (RT-dimension, for short)\cite{zbMATH06253884,DoliwaSZ10,SameiSYZ14}
and more. 

While it is known that some of these measures are tightly linked, the exact relationship between them is still not well understood. In particular, it is a long standing question whether the VC-dimension can be used to give a universal bound on the size of sample compression schemes, 
or on the RT-dimension.

In this work, we make progress on these two questions. First, we prove that the RT-dimension of a boolean concept class $C$ having VC-dimension $d$ is upper bounded by\footnote{In this text $O(f)$ means 
at most $\alpha f + \beta$ for $\alpha,\beta >0$ constants.} $O(d 2^d \log \log |C|)$. Secondly, we give a sample compression scheme of size  $O(d 2^d \log \log |C|)$ that uses additional information.
{Both results were subsequently improved to bounds that are independent of  the size of the concept class $C$~\cite{DBLP:journals/eccc/MoranY15,DBLP:journals/eccc/ChenCT16}}

{Our proofs are based on a similar technique of recursively applying Haussler's Packing Lemma
on the dual class. This similarity provides another example of the informal connection between sample compression schemes and RT-dimension. This connection also appears in other works that study their relationship with the VC-dimension~\cite{DoliwaSZ10,DBLP:journals/eccc/MoranY15,DBLP:journals/eccc/ChenCT16}.}


\subsection{VC-dimension}\label{sec:vc}

\paragraph{{VC-dimension and size.}}
A concept class over the universe $X$ is a set $C\subseteq\{0,1\}^X$.
When $X$ is finite, we denote $|X|$ by $n(C)$.
The VC-dimension of $C$, denoted $\VC(C)$, is the maximum size of a shattered subset of $X$,
where a set $Y \subseteq X$ is shattered if
for every $Z \subseteq Y$ there is $c \in C$ so that $c(x)=1$ for all $x \in Z$ and $c(x)=0$ for all $x \in Y-Z$.

The most basic result concerning VC-dimension is the Sauer-Shelah-Perles Lemma, that upper bounds $|C|$ in terms of $n(C)$ and $\VC(C)$.
It has been independently proved several times, e.g.\ in \cite{zbMATH03392460}.

\begin{theorem}[Sauer-Shelah-Perles]\label{thm:Sauer}
Let $C$ be a boolean concept class with VC-dimension $d$. Then,
$$|C| \leq \sum_{k=0}^{d}{n(C) \choose k}.$$
In particular, if $d\geq 2$ then $|C|\leq n(C)^d$
\end{theorem}

\paragraph{{VC-dimension and PAC learning.}}
The VC-dimension is one of the most basic complexity measures for concept classes. 
It is perhaps mostly known in the context of the PAC learning model. 
PAC learning was introduced in Valiant's seminal work \cite{zbMATH03943062}
as a theoretical model for learning from random examples drawn from an unknown distribution
{(see the book \cite{KearnsVazirani94} for more details)}.

A fundamental and well-known result of Blumer, Ehrenfeucht, Haussler, and Warmuth~\cite{zbMATH04143473},
which is based on an earlier work of Vapnik and Chervonenkis~\cite{zbMATH03391742},
states that {PAC learning sample complexity is equivalent to} VC-dimension.
The proof of this theorem uses Theorem~\ref{thm:Sauer}
and an argument commonly known as double sampling
(see Section~\ref{sec:DoubleSampling} in the appendix for a short
and self contained description of this well known argument).

\begin{theorem}[\cite{zbMATH03391742},\cite{zbMATH04143473}]
\label{thm:PAC}
Let $X$ be a set and $C \subseteq \{0,1\}^X$ be a concept class of VC-dimension $d$.
Let $\mu$ be a distribution over $X$.
Let $\eps,\delta >0$ and $m$ an integer satisfying $2  (2m+1)^{d} (1-\eps/4)^{m}  < \delta$.
Let $c\in C$ and  $Y = (x_1,\dots,x_m)$ 
be a multiset of $m$ independent samples from $\mu$.
Then, the probability that there is $c' \in C$
so that $c|_Y = c'|_Y$ but
$\mu(\{x : c(x) \neq c'(x)\}) > \eps$ is at most $\delta$.
\end{theorem}


\paragraph{{VC-dimension and the metric structure.}}
Another fundamental result in this area is Haussler's \cite{zbMATH00734534}
description of the metric structure of concept classes with low VC-dimension 
(see also the work of Dudley \cite{zbMATH03628097}).
Roughly, it says that a concept class $C$ of VC-dimension $d$, when thought of as an $L_1$
metric space, behaves like a $d$ dimensional space in the sense that the size of an $\eps$-separated set in $C$ is at most $(1/\eps)^d$.
More formally, every probability distribution $\mu$ on $X$ induces the (pseudo) metric
$$\dist_\mu(c,c') = \mu( \{x : c(x) \neq c'(x)\})$$
on $C$.
A set $S  \subseteq C$ is called $\eps$-separated with respect to $\mu$
if for every two concepts $c \neq c'$ in $S$ 
we have $\dist_\mu(c,c') > \eps$.
A set $A = A_\mu(C,\eps) \subseteq C$ is called 
an {\em $\eps$-approximating set}\footnote{In metric spaces such a set is called
an $\eps$-net, however in learning theory and combinatorial geometry the term $\eps$-net has a different meaning,
so we use $\eps$-approximating instead.} for $C$ with respect to $\mu$
if it is a maximal $\eps$-separated set with respect to $\mu$.
The maximality of $A$ implies that for every $c\in C$ there is some 
rounding $r = r(c,\mu,C,\eps)$ in $A$ 
so that $r$ is a good approximation to $c$,
that is, $\dist_\mu(c,r) \leq \eps$.
We call $r$ a rounding of $c$ in $A$.

An approximating set can be thought of as a metric approximation of
the possibly complicated concept class $C$,
and for many practical purposes it is a good enough substitute for $C$.
Haussler proved that there are always small approximating sets.

\begin{theorem}[Haussler]
\label{thm:Haussler}
Let $C \subseteq \{0,1\}^X$ be a concept class with VC-dimension $d$. 
Let $\mu$ be a distribution on $X$.
Let $\eps\in (0,1]$.
If $S$ is $\eps$-separated with respect to $\mu$ then
$$|S| \leq e (d+1) \left(\frac{2e}{\eps}\right)^d
\leq \left(\frac{4e^2}{\eps}\right)^d.$$
\end{theorem}

\begin{proof}[A proof of a weaker statement]
For  {$m =  2 \log(|S|)/\eps$},  let $x_1,\ldots,x_m$ be
independent samples from $\mu$.
For every $c \neq c'$ in $S$,
$$\Pr_{\mu^m} \left( \forall i \in [m] \ \ c(x_i) = c'(x_i) \right)
< (1-\eps)^m \leq e^{-m \eps} \leq 1/|S|^2.$$
The union bound implies that
there is a choice of $Y \subseteq X$ of size
$|Y| \leq m$ so that $|S|_Y| = |S|$.
Theorem~\ref{thm:Sauer} implies $|S| \leq (|Y|+1)^d$. 
Thus, $|S| < \left( 30 d \log(2 d/\eps) / \eps \right)^d$.
\end{proof}

\subsection{Teaching}
 
Imagine a teacher that helps a student to learn a concept $c$ by picking insightful examples. 
The concept $c$ is known only to the teacher, but $c$ belongs to a class of concepts $C$ known to both the teacher and the student. 
The teacher carefully chooses a set of examples that is tailored for $c$, and then provides these examples to the student. Now, the student should be able to recover $c$ from these examples. 

A central issue that is addressed in the  
design of mathematical teaching  models is ``collusions.''
Roughly speaking, a collusion occurs when the teacher and the student
agree in advance on some unnatural encoding of information about $c$ using the 
bit description of the chosen examples, instead of using attributes that separate $c$ from other concepts.
Many mathematical models for teaching were suggested: 
Shinohara and Miyano~\cite{ShinoharaM90}, Jackson and Tomkins~\cite{JacksonT92}, Goldman, Rivest and Schapire~\cite{GoldmanRS93},
Goldman and Kearns~\cite{GoldmanK95}, Goldman and Mathias~\cite{GoldmanM96} Angluin and Krikis~\cite{AngluinK03}, Balbach~\cite{Balbach2007}, and Kobayashi and Shinohara~\cite {KobayashiS09}.
We  now discuss some of these models in more detail.

\paragraph{{Teaching sets.}}
The first mathematical models for teaching~\cite{GoldmanK95,ShinoharaM90,AnthonyBCS92}
handle collusions in a fairly restrictive way,
by requiring that the teacher provides a set of examples $Y$ 
that uniquely identifies $c$.
Formally, this is captured by the notion of a teaching set, which was independently  
introduced by Goldman and Kearns~\cite{GoldmanK95},
Shinohara and Miyano~\cite{ShinoharaM90} and Anthony et al.~\cite{AnthonyBCS92}.
A set $Y\subseteq X$ is a teaching set for $c$ in $C$ if for all $c' \neq c$ in $C$, we have $c'|_Y \neq c|_Y$.
The teaching complexity in these
models is captured by the hardest concept to teach, i.e., $\max_{c \in C}\min\{ |Y| \,:\, Y \text{ is a teaching set for } c \text{ in } C \}$.

Teaching sets also appear in other areas of learning theory: Hanneke~\cite{Hanneke07} used it in his study of the label complexity in active learning, and the authors of~\cite{WigdersonY12} used variants of it to design efficient algorithms for learning distributions
using imperfect data.

Defining the teaching complexity using the hardest concept
is often too restrictive.
Consider for example the concept class consisting of all singletons and the empty set over a domain $X$ of size $n$. Its teaching complexity in these models is $n$, since the only teaching set for the empty set is $X$. 
This is a fairly simple concept class that has the maximum possible complexity.

%
%
 
\paragraph{{Recursive teaching dimension.}}
Goldman and Mathias~\cite{GoldmanM96} and Angluin and Krikis~\cite{AngluinK03}
therefore suggested less restrictive teaching models,
and more efficient teaching schemes were indeed discovered in these models.
One approach, studied by Zilles et al.~\cite{zbMATH06253884}, Doliwa et al.~\cite{DoliwaSZ10}, and Samei et al.~\cite{SameiSYZ14}, uses a natural hierarchy
on the concept class $C$ which is defined as follows. The first layer in the hierarchy consists of all concepts whose teaching set has minimal size. Then, these concepts are removed and the second layer consists of all concepts whose teaching set with respect to the remaining concepts has minimal size. Then, these concepts are removed and so on, until all concepts are removed. The maximum size of a 
set that is chosen in this process is called the {\it recursive teaching (RT)} dimension. 
{One way of thinking about this model is that the teaching
process satisfies an Occam's razor-type rule 
of preferring simpler concepts.}
For example, the concept class consisting of singletons and the empty set, which was considered earlier, has recursive teaching dimension $1$: The first layer in the hierarchy consists
of all singletons, which have teaching sets of size $1$.
Once all singletons are removed, we are left with a concept class
of size $1$, the concept class $\{\emptyset\}$,
and in it the empty set has a teaching set of size $0$.

A similar notion to RT-dimension was independently suggested in~\cite{WigdersonY12} under the terminology of partial IDs. There the focus was on getting a simultaneous upper bound on the size of the sets, as well as the number of layers in the recursion, and it was shown that for any concept class $C$ both can be made at most $\log|C|$. Motivation for this study comes from the population recovery learning problem defined in~\cite{DRWY12}.

\paragraph{{Previous results.}}



Doliwa et al.~\cite{DoliwaSZ10} and Zilles et al.~\cite{zbMATH06253884} asked whether small VC-dimension implies small recursive teaching dimension. An equivalent question was asked 10 years earlier by Kuhlmann~\cite{Kuhlmann99}.  Since the VC-dimension does not increase when concepts are removed from the class, this question is equivalent to asking whether every class with small VC-dimension has some concept in it with a small teaching set. Given the semantics of the recursive teaching dimension and the VC-dimension, an interpretation of this question
is whether  exact teaching is not much harder than approximate learning (i.e., PAC learning).

For infinite classes the answer to this question is negative. There is an infinite concept class with VC-dimension $1$ so that every concept in it does not have a finite teaching set. An example for such a class is $C\subseteq\{0,1\}^\mathbb{Q}$ defined as $C=\{c_q:q\in\mathbb{Q}\}$ where $c_q$ is the indicator function of all rational numbers that are smaller than $q$. The VC-dimension of $C$ is $1$, but every teaching set for some $c_q \in C$ must contain a sequence of rationals that converges to $q$. 

For finite classes this question is open.
However, in some special cases it is known that the answer is affirmative.
In~\cite{Kuhlmann99} it is shown that if $C$ has VC-dimension $1$, then its recursive teaching dimension is also $1$. It is known that if $C$ is a maximum\footnote{That is, $C$ satisfies Sauer-Shelah-Perles Lemma with equality.} class then its recursive teaching dimension is equal to its VC-dimension~\cite{DoliwaSZ10,DBLP:journals/jmlr/RubinsteinR12}. Other families of concept classes for which the recursive teaching dimension is at most the VC-dimension are discussed in~\cite{DoliwaSZ10}.
In the other direction, \cite{Kuhlmann99} provided examples of concept classes with VC-dimension $d$ and recursive teaching dimension at least $\frac{3}{2}d$.

The only bound on the recursive teaching dimension for general classes 
was observed by both~\cite{DoliwaSZ10,WigdersonY12}. It states that the recursive teaching dimension of $C$ is at most $\log|C|$. This bound follows from a simple halving argument which shows that for all $C$ there exists some $c\in C$ with a teaching set of size $\log|C|$. 

\paragraph{{Our contribution.}}
Our first main result is the following general bound, which exponentially improves over the $\log|C|$ bound when the VC-dimension is small
(the  proof is given in Section~\ref{sec:RTD}).

\begin{theorem}[RT-dimension]
\label{thm:RTD}
Let $C$ be a concept class of
VC-dimension $d$. Then there exists $c\in C$
with a teaching set of size at most $$d2^{d+3}(\log(4e^2) + \log\log|C|) .$$
\end{theorem}

It follows that the recursive teaching dimension of concept classes
of VC-dimension $d$ is at most 
$d2^{d+3}(\log(4e^2) + \log\log|C|)$
as well.

{Subsequent to this paper, Chen, Cheng, and Tang~\cite{DBLP:journals/eccc/ChenCT16} proved that the RT-dimension is at most $\exp(d)$. Their proof is based on ideas from this work, in particular they follow and improve the argument from the proof of Lemma~\ref{lem:3,6}.}

\subsection{{Sample} compression schemes}
\label{sec:intCompSch}

A fundamental and well known statement in learning theory says that if the VC-dimension of a concept class $C$ is small, then any consistent\footnote{An algorithm that outputs an hypothesis in $C$ that is consistent with the input examples.} algorithm successfully PAC learns concepts from $C$ after seeing just a few labelled examples~\cite{zbMATH03391742,BlumerEHW87}.
In practice, however, a major challenge one has to face when designing a learning algorithm is the construction of an hypothesis that is consistent with the examples seen.
Many learning algorithms share the property that the output hypothesis
is constructed using a small subset of the examples. For example, in support vector machines, only the set of support vectors is needed to construct the separating hyperplane \cite{Cristianini00a}.
{Sample compression schemes provide a formal 
meaning for this algorithmic property.}

Before giving the formal definition of compression schemes, let us consider
a simple illustrative example.
Assume we are interested in learning the concept class of intervals
on the real line.  We get a collection of 100 samples of the form
$(x,c_I(x))$ where $x \in \R$ and $c_I(x) \in \{0,1\}$ indicates\footnote{That is $c_I(x)=1$ iff $x \in I$.} if $x$ is
in the interval $I \subset \R$.
Can we remember just a few of the samples
in a way that allows to recover all the 100 samples?  
In this case, the answer is affirmative and in fact it is easy to do so.
Just remember two locations, those of the left most $1$ and of the right most $1$
(if there are no $1$s, just remember one of the $0$s).
From this data, we can reconstruct the value of $c_I$ on all the other 
100 samples.


\paragraph{The formal definition.}
Littlestone and Warmuth~\cite{littleWarm} formally defined
sample compression schemes as follows.
Let $C \subseteq \{0,1\}^X$ with $|X|=n$.
Let
$$L_C(k_1,k_2) = \{(Y,y) : Y \subseteq X , \ k_1 \leq |Y| \leq k_2,  \ y \in C|_Y\},$$
the set of labelled samples from $C$, of sizes between $k_1$ and $k_2$.
A $k$-sample compression scheme for $C$ with information $Q$, consists 
of two maps $\kappa,\rho$ for which the following hold:

\begin{description}

\item[(${\kappa}$)] 
The {\em compression map}
$$\kappa: L_C(1,n) \to L_C(0,k) \times Q$$ takes
$(Y,y)$ to $((Z,z),q)$ with $Z \subseteq Y$
and $y|_Z = z$.

\item[($\rho$)] The {\em reconstruction map}
$$\rho : L_C(0,k) \times Q \to \{0,1\}^X$$
is so that for all $(Y,y)$ in $L_C(1,n)$,
$$\rho(\kappa(Y,y))|_Y = y.$$
{The {\em size} of the scheme is  $k+\log|Q|$.}
\end{description}
Intuitively, the compression map takes a long list of samples
$(Y,y)$
and encodes it as a short sub-list of samples $(Z,z)$
together with some small amount of side information $q\in Q$,
which helps in the reconstruction phase.
The reconstruction takes a short list of samples $(Z,z)$ and decodes it
using the side information $q$, {without any knowledge
of $(Y,y)$},
to an hypothesis in a way that essentially inverts the compression.
{Specifically, the following property must always hold:} if the compression of $(Y,c|_Y)$ is the same as that of $(Y',c'|_{Y'})$ then $c|_{Y\cap Y'} = c'|_{Y\cap Y'}$.

A different perspective of the
side information is as a list decoding
in which the small set of labelled examples $(Z,z)$
is mapped to the set of hypothesis $\{ \rho((Z,z),q) : q \in Q\}$,
one of which is correct.

We note that it is not necessarily the case that the reconstructed hypothesis belongs to the original class $C$. All it has to satisfy is that for any $(Y,y)\in L_C(1,n)$ such that $h=\rho(\kappa(Y,y))$ we have that $h|_Y = y$. Thus, $h$ has to be consistent only on the sampled coordinates that were compressed and not elsewhere. 

{Let us consider a simple example of a sample compression scheme, 
to help digest the 
definition.
Let $C$ be a concept class and
let $r$ be the rank over, say, $\R$
of the matrix whose rows correspond to the concepts in $C$. 
We claim that there is an $r$-sample compression
scheme for $C$ with no side information.
Indeed, for any $Y \subseteq X$,
let $Z_Y$ be a set of at most $r$ columns
that span the columns of the matrix $C|_Y$.
Given a sample $(Y,y)$ compress it to
$\kappa(Y,y) = (Z_Y,z)$ for $z = y|_{Z_Y}$. 
The reconstruction maps $\rho$ takes $(Z,z)$
to any concept $h \in C$ so that $h|_Z = z$.
This sample compression scheme works since
if $(Z,z) =\kappa(Y,y)$
then every two different rows in $C|_Y$ must 
disagree on $Z$.
} 

\paragraph{Connections to learning.}
{Sample compression schemes are known to yield practical learning algorithms 
(see e.g.\ \cite{DBLP:journals/jmlr/MarchandS02}),
and allow learning for multi labelled concept classes~\cite{DBLP:conf/colt/SameiSYZ14}.

They} can also be interpreted
as a formal manifestation of Occam's razor.
Occam's razor is a philosophical principle attributed to William of Ockham from the late middle ages.
It says that in the quest for an explanation or an hypothesis, one should prefer the simplest one which is consistent with the data. There 
are many works on the role of Occam's razor in
learning theory, a partial list includes  \cite{littleWarm,BlumerEHW87,DBLP:conf/colt/Floyd89,DBLP:journals/iandc/QuinlanR89,DBLP:journals/jcss/HelmboldW95,DBLP:journals/ml/FloydW95,DBLP:journals/datamine/Domingos99}.
In the context of sample compression schemes, 
simplicity is captured by the size of the compression scheme.
Interestingly, this manifestation of Occam's razor is provably useful \cite{littleWarm}: 
Sample compression schemes imply {PAC} learnability.

\begin{theorem}[Littlestone-Warmuth]
\label{thm:LWPAC}
Let $C \subseteq \{0,1\}^Xß$, and $c \in C$.
Let $\mu$ be a distribution on $X$, and
 $x_1,\ldots,x_m$ be $m$ independent samples from $\mu$.
Let $Y = (x_1,\ldots,x_m)$ and $y = c|_Y$.
Let $\kappa,\rho$ be a $k$-sample compression scheme
for $C$ with additional information $Q$.
Let $h = \rho(\kappa(Y,y))$.
Then,
$$\Pr_{\mu^m}( \dist_\mu(h,c) > \eps)
<|Q| \sum_{j=0}^{k} {m \choose j} (1-\eps)^{m-j}.$$
\end{theorem}

\begin{proof}[Proof sketch.]
There are $\sum_{j=0}^{k} {m \choose j}$
subsets $T$ of $[m]$ of size at most $k$.
There are $|Q|$ choices for $q \in Q$.
Each choice of $T,q$ yields a function $h_{T,q}
= \rho((T,y_T),q)$ that is measurable with respect to $x_T = (x_t : t \in T)$.
The function $h$ is one of the functions in $\{h_{T,q} : |T|\leq k,q\in Q\}$.
For each $h_{T,q}$, the coordinates in $[m] - T$
are independent,
and so if $\dist_\mu(h_{T,q},c) > \eps$
then the probability that all these $m-|T|$ samples
agree with $c$ is less than $(1-\eps)^{m-|T|}$.
The union bound completes the proof.
\end{proof}

{
The sample complexity of PAC learning is essentially the VC-dimension.
Thus, from Theorem~\ref{thm:LWPAC}
we expect the VC-dimension to bound from below the size of sample compression schemes.
Indeed, \cite{DBLP:journals/ml/FloydW95} proved
that there are concept classes of VC-dimension $d$ 
for which any sample compression scheme has size at least $d$.}

%

{
This is part of the motivation for the following basic question that was
asked by Littlestone and Warmuth \cite{littleWarm} nearly 30 years ago:
Does a concept class of VC-dimension $d$ have a sample
compression scheme of size depending only on $d$ (and not on the universe size)?}

In fact, unlike the VC-dimension, the definition of sample compression schemes
as well as the fact that they imply PAC learnability naturally generalizes to multi-class classification problems~\cite{DBLP:conf/colt/SameiSYZ14}. 
{Thus, Littlestone and Warmuth's question above 
can be seen as the boolean instance of a much broader question:
Is it true that the size of an optimal sample compression scheme for a given concept class (not necessarily binary-labeled) is the sample complexity of PAC learning of this class?
}

\paragraph{{Previous constructions.}}
{
Floyd~\cite{DBLP:conf/colt/Floyd89}
and
Floyd and Warmuth~\cite{DBLP:journals/ml/FloydW95} 
constructed 
sample compression schemes of size $\log |C|$.
The construction in \cite{DBLP:journals/ml/FloydW95} uses a transformation that converts
certain 
online learning algorithms to compression schemes.}
Helmbold and Warmuth~\cite{DBLP:journals/jcss/HelmboldW95} 
and Freund~\cite{DBLP:journals/iandc/Freund95} showed how to compress a sample of size $m$ to a sample of size $O(\log(m))$ using some side information for classes of constant VC-dimension
(the implicit constant in the $O(\cdot)$ depends on the VC-dimension).

{In a long line of works, several interesting compression schemes
for special cases were constructed. A partial list includes
Helmbold et al.\ \cite{DBLP:journals/siamcomp/HelmboldSW92},
Floyd and Warmuth \cite{DBLP:journals/ml/FloydW95},
Ben-David and Litman \cite{DBLP:journals/dam/Ben-DavidL98},
Chernikov and Simon \cite{chernikovS},
Kuzmin and Warmuth \cite{DBLP:journals/jmlr/KuzminW07},
Rubinstein et al.\ \cite{DBLP:journals/jcss/RubinsteinBR09}, Rubinstein and Rubinstein \cite{DBLP:journals/jmlr/RubinsteinR12},
Livni and Simon \cite{DBLP:conf/colt/LivniS13}
and more.
These works provided connections
between compression schemes and geometry, topology and model theory.
}

\paragraph{{Our contribution.}}
Here we make the first quantitive progress on this question, 
since the work of Floyd~\cite{DBLP:conf/colt/Floyd89}.
The following theorem shows that low VC-dimension implies
the existence of relatively efficient compression schemes.
The constructive proof is provided in Section~\ref{sec:compSch}.


\begin{theorem}[Sample compression scheme]
\label{thm:CompSch}
If $C$ has VC-dimension $d$ then it has
a 
$k$-sample compression scheme with additional information $Q$
where
$k =O(d 2^d  \log\log|C|)$ and
$\log|Q| \leq O(k \log(k) )$.
\end{theorem}

{Subsequent to this paper, the first and the last authors improved this bound~\cite{DBLP:journals/eccc/MoranY15},
showing that any concept class of VC-dimension $d$ has a sample compression scheme of size at most $\exp(d)$.
The techniques used in~\cite{DBLP:journals/eccc/MoranY15} differ from the techniques we use
in this paper. {In particular, our scheme relies on Haussler's Packing Lemma (Theorem~\ref{thm:Haussler}) 
and recursion, while the scheme in~\cite{DBLP:journals/eccc/MoranY15} relies on von Neumann's minimax theorem~\cite{Neumann1928} and the $\eps$-approximation theorem~\cite{zbMATH03391742,DBLP:journals/dcg/HausslerW87}, 
which follow from the double-sampling argument of~\cite{zbMATH03391742}.}
Thus, despite the fact that our scheme is weaker than the one in~\cite{DBLP:journals/eccc/MoranY15},
it provides a different angle on sample compression, 
which may be useful in further improving
the exponential dependence on the VC-dimension 
to an optimal linear dependence,
as conjectured by Floyd and Warmuth~\cite{DBLP:journals/ml/FloydW95,DBLP:conf/colt/Warmuth03}.}

%

\subsection{Discussion and open problems}

This work provides relatively efficient constructions
of teaching sets and sample compression schemes.
{However, the exact relationship between VC-dimension, sample compression scheme size, and the RT-dimension remains unknown.}
Is there always a concept with a teaching set 
of size depending only on the  VC-dimension?
(The interesting case is finite concept classes,
as mentioned above.)
Are there always sample compression schemes of size
linear (or even polynomial) in the VC-dimension?

The simplest case that is still open is VC-dimension $2$. 
One can refine this case even further. 
VC-dimension $2$ means that on any 
three coordinates $x,y,z\in X$, the projection $C|_{\{x,y,z\}}$ has at most $7$ patterns.
A more restricted family of classes is $(3,6)$ concept classes, for which on any three coordinates there are at most $6$ patterns. 
We can show that the recursive
teaching dimension of $(3,6)$ classes is at most $3$. 
\begin{lemma}\label{lem:3,6}
Let $C$ be a finite $(3,6)$ concept class. Then there exists
some $c\in C$ with a teaching set of size at most $3$.
\end{lemma}
\begin{proof}
Assume that $C\subseteq \{0,1\}^{X}$ with $X = [n]$.
If $C$ has VC-dimension $1$
then there exists $c\in C$ with a teaching set of size $1$
(see~\cite{Kuhlmann99,AMY14}).
Therefore, assume that the VC-dimension of $C$ is $2$.
Every shattered pair $\{x,x'\}\subseteq X$ partitions
$C$ to $4$ nonempty sets:
$$C^{x,x'}_{b,b'} = \{c\in C: c(x)=b,c(x')=b'\},$$
for $b,b'\in\{0,1\}$.
Pick a shattered pair $\{x_*,x'_*\}$ and $b_*,b'_*$ for which
the size of $C^{x_*,x'_*}_{b_*,b'_*}$ is minimal.
Without loss of generality assume that 
$\{x_*,x'_*\}=\{1,2\}$ and that $b_*=b'_*=0$.
To simplify notation, we denote $C^{1,2}_{b,b'}$ 
simply by $C_{b,b'}$.

We prove below that $C_{0,0}$ has VC-dimension $1$.
This completes the proof since then there is some $c\in C_{0,0}$ and some $x \in [n] \setminus \{1,2\}$
such that $\{x\}$ is a teaching set for $c$ in $C_{0,0}$.
Therefore, $\{1,2,x\}$ is a teaching set for $c$ in $C$.

First, a crucial observation is that since $C$ is a $(3,6)$ class, no pair $\{x,x'\}
\subseteq [n]\setminus\{1,2\}$  
is shattered by both $C_{0,0}$ and $C\setminus C_{0,0}$. Indeed, 
if $C\setminus C_{0,0}$ shatters $\{x,x'\}$
then either $C_{1,0}\cup C_{1,1}$ or
$C_{0,1}\cup C_{1,1}$ has at least $3$ patterns on $\{x,x'\}$.
If in addition $C_{0,0}$ shatters $\{x,x'\}$
then $C$ has at least $7$ patterns on $\{1,x,x'\}$ or $\{2,x,x'\}$, contradicting the assumption that $C$ is a $(3,6)$ class.

Now,
assume towards contradiction that $C_{0,0}$ shatters $\{x,x'\}$.
Thus, $\{x,x'\}$ is not shattered by $C\setminus C_{0,0}$ which means that there is some
pattern $p \in \{0,1\}^{\{x,x'\}}$ so that $p \not \in (C\setminus C_{0,0})|_{\{x,x'\}}$.
This implies that $C^{x,x'}_{p(x),p(x')}$ is a proper subset of $C_{0,0}$, contradicting
the minimality of $C_{0,0}$.
\end{proof}

\section{The dual class}

We shall repeatedly use the dual concept class to $C$ and its properties.
The dual concept class $C^*\subseteq\{0,1\}^C$ of $C$ is defined by $C^*=\{c_x:x\in X\}$, 
where $c_x:C\rightarrow\{0,1\}$ is the map so that $c_x(c)=1$ iff $c(x)=1$.
If we think of $C$ as a binary matrix whose rows
are the concepts in $C$, then
$C^*$ corresponds to the distinct rows of the transposed matrix
(so it may be that $|C^*| < |n(C)|$).

%

We use the following well known property (see \cite{Assouad}).

\begin{claim}[Assouad]
\label{clm:assou}
If the VC-dimension of $C$ is $d$ then the VC-dimension of $C^*$ is at most $2^{d+1}$.
\end{claim}

\begin{proof}[Proof sketch]
If the VC-dimension of $C^*$ is $2^{d+1}$
then in the matrix representing $C$ there are $2^{d+1}$
rows that are shattered, and in these rows
there are $d+1$ columns that are shattered.
\end{proof}

We also define the dual approximating set (recall the definition of $A_\mu(C,\eps)$ from Section~\ref{sec:vc}).
Denote by $A^*(C,\eps)$ the set $A_U(C^*,\eps)$, where $U$ is the uniform distribution on $C^*$.

\section{Teaching sets}
\label{sec:RTD}

In this section we prove Theorem~\ref{thm:RTD}. The high level idea is to use Theorem~\ref{thm:Haussler}
and Claim~\ref{clm:assou} to identify two distinct $x,x'$ in $X$ so that
the set of $c \in C$ so that $c(x) \neq c(x')$ is much smaller than $|C|$,
add $x,x'$ to the teaching set, and continue inductively.

\begin{proof}[Proof of Theorem~\ref{thm:RTD}]
For classes with VC-dimension $1$ there is $c\in C$ with a
teaching set of size $1$, see e.g.\
\cite{DoliwaSZ10}.
We may therefore assume that $d\geq 2$.

We show that if $|C| > (4e^2)^{d\cdot 2^{d+2}}$, then there exist
$x \neq x'$ in $X$ such that 
\begin{equation}\label{eq:RTDstep}
0< |\{c\in C : c(x)=0 \text{ and } c(x')=1\} | \leq |C|^{1-\frac{1}{d2^{d+2}}}.
\end{equation}
From this the theorem follows, since if we iteratively add such $x,x'$ to the teaching
set {and restrict ourselves to $\{c\in C : c(x)=0 \text{ and } c(x')=1\}$, then after at most $d2^{d+2}\log\log|C|$ iterations, the size of the remaining class} is reduced to less than $(4e^2)^{d\cdot2^{d+2}}$. At this point we can identify a unique concept by adding at most $\log((4e^2)^{d\cdot 2^{d+2}})$ additional indices to the teaching set, using the halving argument of \cite{DoliwaSZ10,WigdersonY12}.
This gives a teaching set of size at most
$2d2^{d+2}\log\log|C| + d2^{d+2}\log(4e^2)$ for some $c\in C$, as required.

In order to prove~\eqref{eq:RTDstep}, it is enough to
show that there exist $c_x \neq c_{y}$ in $C^*$
such that the normalized hamming
distance between $c_x,c_{y}$ is at most $\eps := |C|^{-\frac{1}{d2^{d+2}}}$.
%
%
Assume towards contradiction that the distance between every two concepts in
$C^*$ is more than $\eps$, and assume without loss of generality
that $n(C)=|C^*|$ (that is, all the columns in $C$ are distinct).
By Claim~\ref{clm:assou}, the VC-dimension of $C^*$ is at most $2^{d+1}$.
Theorem~\ref{thm:Haussler} thus implies that
\begin{align}
\label{eqn:RisSmallTeach}
n(C)= |C^*| \leq \left(\frac{4e^2}{\eps}\right)^{2^{d+1}}
< \left(\frac{1}{\eps}\right)^{2^{d+2}},
\end{align}
where the last inequality follows from the definition of $\eps$ and the assumption on the size of $C$.
Therefore, we arrive at the following contradiction:
\begin{align*}
|C| & \leq (n(C))^d \tag{by Theorem~\ref{thm:Sauer}, since $VC(C)\geq 2$ } \\
      & < \left(\frac{1}{\eps}\right)^{d\cdot2^{d+2}}\tag{by Equation~\ref{eqn:RisSmallTeach} above}\\
      & = |C| . \tag{by definition of $\eps$}
\end{align*}
\end{proof}

\section{Sample compression schemes}
\label{sec:compSch}

In this section we prove Theorem~\ref{thm:CompSch}.
The theorem statement and the definition 
of sample compression schemes appear in Section~\ref{sec:intCompSch}.

While the details are somewhat involved, due to the complexity of the definitions, the high level idea may be (somewhat simplistically) summarized as follows.

For an appropriate choice of $\epsilon$, we pick an $\eps$-approximating set $A^*$ of the dual class $C^*$. It is helpful to think of $A^*$ as a subset of the domain $X$.
Now, either $A^*$ faithfully represents the sample $(Y,y)$ or it does not
(we do not formally define ``faithfully represents'' here).
We identify the following win-win situation: 
In both cases, we can reduce the compression task to that in a much smaller set of concepts 
of size at most $\eps |C| \approx |C|^{1-2^{-d}}$,
similarly to as for teaching sets in Section~\ref{sec:RTD}.
This yields the same double-logarithmic behavior.

In the case that $A^*$ faithfully represents $(Y,y)$, Case \ref{kappa:2} below, we recursively compress in the small class $C|_{A^*}$.
In the unfaithful case, Case \ref{kappa:1} below, we recursively compress in a (small) set of concepts for which disagreement occurs on some point of $Y$, just as in Section~\ref{sec:RTD}.
In both cases, we have to extend the recursive solution, and the cost is adding one sample point to the compressed sample (and some small amount of additional 
information by which we encode whether Case~\ref{kappa:1} or~\ref{kappa:2} occurred).

The compression we describe is inductively defined,
and has the following additional structure.
Let $((Z,z),q)$ be in the image of $\kappa$. The information $q$ is of the form
$q=(f,T)$, 
where $T \geq 0$ is an integer so that $|Z| \leq T+O(d \cdot 2^d)$, 
and $f : \{0,1,\ldots,T\} \to Z$ is
a partial one-to-one function\footnote{That is,
it is defined over a subset of $\{0,1,\ldots,T\}$
and it is injective on its domain.}. 

The rest of this section is organized as follows. In Section~\ref{sec:kappa} we define the compression map $\kappa$. In Section~\ref{sec:rho} we give the reconstruction map $\rho$. The proof of correctness is given in Section~\ref{sec:correct} and the upper bound on the size of the compression is calculated in Section~\ref{sec:size}.

\subsection{Compression map: defining $\kappa$}\label{sec:kappa}
Let $C$ be a concept class.
The compression map is defined by induction on $n=n(C)$.
For simplicity of notation, let $d = VC(C)+2$. 

In what follows we shall routinely use $A^*(C,\epsilon)$. There are several $\eps$-approximating sets and so we would like to fix one of them, say, the one obtained by greedily adding columns to $A^*(C,\eps)$ starting from the first\footnote{We shall assume w.l.o.g. that there is some well known order on $X$.} column (recall  that we can think of $C$ as a matrix whose rows correspond to concepts in $C$ and whose columns are concepts in the dual class $C^*$). To keep notation simple, we shall use $A^*(C,\eps)$ to denote both the approximating set in $C^*$ and the subset of $X$ composed of columns that give rise to $A^*(C,\epsilon)$. This is a slight abuse of notation but the relevant meaning will always be clear from the context.


\paragraph{Induction base.}
The base of the induction applies to all concept classes $C$ so that 
$|C| \leq (4e^2)^{d\cdot2^d + 1}$. 
In this case, we use the compression scheme
of Floyd and Warmuth~\cite{DBLP:conf/colt/Floyd89,DBLP:journals/ml/FloydW95} which has size $\log(|C|) = O(d \cdot 2^d)$.
This compression scheme has no additional information. Therefore, to maintain the structure of our compression scheme we append to it redundant additional information by setting $T=0$ and $f$ to be empty.

\paragraph{Induction step.}
Let $C$ be so that $|C| > (4e^2)^{d\cdot 2^d + 1}$.
Let $0 < \eps <1$ be so that
\begin{align}
\label{en:whatIsEps}
\eps |C| = \left( \frac{1}{\eps} \right)^{d\cdot2^{d}} .
\end{align}
This choice 
balances the recursive size.
By Claim~\ref{clm:assou}, the VC-dimension of $C^*$ is at most $2^{d-1}$ 
(recall that $d=VC(C)+2$).
Theorem~\ref{thm:Haussler} thus implies that
\begin{align}
\label{eqn:RisSmall}
|A^*(C,\eps)| \leq \left(\frac{4e^2}{\eps}\right)^{2^{d-1}}
< \left(\frac{1}{\eps}\right)^{2^{d}} < n(C).
\end{align}
(Where the second inequality follows from the definition of $\eps$
and the assumption on the size of $C$ and the last inequality follows from the definition of 
$\eps$ and Theorem~\ref{thm:Sauer}).

Let $(Y,y) \in L_C(1,n)$.
Every $x \in X$ has a rounding\footnote{The choice of $r(x)$ also depends on $C,\eps$,
but to simplify the notation we do not explicitly mention it.}
$r(x)$ in $A^*(C,\eps)$.
We distinguish between two cases:

\begin{enumerate}[\bf{Case} 1:] 
\item \label{kappa:1}
There exist $x\in Y$ and $c\in C$ such that $c|_Y=y$ and $c(r(x)) \neq c(x)$.

This is the unfaithful case in which we {recurse} as in
Section~\ref{sec:RTD}.
Let 
\begin{align*}
& C' = \{ c'|_{X-\{x,r(x)\}}:c'\in C, c'(x)=c(x),c'(r(x))=c(r(x)) \} ,\\
& Y' = Y-\{x,r(x)\},\\
& y' = y|_{Y'}.
\end{align*}
Apply recursively $\kappa$ on $C'$ and the sample $(Y',y')\in L_{C'}(1,n(C'))$.
Let $((Z',z'),(f',T'))$ be the result of this compression.
Output $((Z,z),(f,T))$ defined as\footnote{Remember
that $f$ is a partial function.}
\begin{align*}
& Z=Z'\cup\{x\} , &\\ 
& z|_{Z'}=z' ,\ z(x)=y(x), &\\ 
& T=T'+1 , &\\
&f|_{\{0,\ldots,T-1\}}=f'|_{\{0,\ldots,T-1\}}, &\\
& f(T)=x \tag{$f$ is defined on $T$, marking that Case~\ref{kappa:1} occurred} 
\end{align*}

\item  \label{kappa:2}
For all $x\in Y$ and $c\in C$ such that $c|_Y=y$,
we have $c(x)= c(r(x))$.

This is the faithful case,
in which we compress by restricting $C$ to 
$A^*$.
Consider $r(Y)=\{r(y'): y'\in Y\} \subseteq A^*(C,\eps)$.
For each $x'\in r(Y)$, pick\footnote{The function $s$
can be thought of as the inverse of $r$.
Since $r$ is not necessarily invertible we use a different
notation than $r^{-1}$.} $s(x')\in Y$ to be an element such that $r(s(x'))=x'$. 
Let
\begin{align*}
& C' = C|_{A^*(C,\eps)},\\
& Y' = r(Y),\\
& y'(x') = y(s(x')) \ \forall x'\in Y' .
\end{align*}
By \eqref{eqn:RisSmall}, we know $|A^*(C,\eps)| < n(C)$.
Therefore, we can recursively apply $\kappa$ on $C'$ and $(Y',y')\in L_{C'}(1,n(C'))$ and get $((Z',z'),(f',T'))$.
Output $((Z,z),(f,T))$ defined as
\begin{align*}
& Z=\{s(x'):x'\in Z'\}, \\ 
& z(x) = z'(r(x)) \ \forall x\in Z , \tag{$r(x)\in Z'$}\\ 
& T=T'+1 , \\
& f=f'. \tag{$f$ is not defined on $T$, marking that Case~\ref{kappa:2} occurred}
\end{align*}
\end{enumerate}

The following lemma summarizes two key properties of the compression scheme.
The correctness of this lemma follows directly from the definitions of Cases~\ref{kappa:1} and~\ref{kappa:2} above.

\begin{lemma}
\label{property:compScheme}
Let $(Y,y) \in L_C(1,n(C))$ and $((Z,z),(T,f))$ be
the compression of $(Y,y)$ described above,
where $T\geq 1$.
The following properties hold:
\begin{enumerate}
\item \label{prop:defT}
$f$ is defined on $T$ and $f(T)=x$ iff
$x\in Y$ and there exists $c\in C$ such that $c|_Y=y$ and $c(r(x)) \neq c(x)$.
\item \label{prop:notdefT}
$f$ is not defined on $T$ iff
for all $x\in Y$ and $c\in C$ such that $c|_Y=y$, it holds that $c(x)= c(r(x))$.
\end{enumerate}
\end{lemma}

\subsection{Reconstruction map: defining $\rho$}\label{sec:rho}

The reconstruction map is similarly defined by induction on $n(C)$.
Let $C$ be a concept class and let $((Z,z),(f,T))$ 
be in the image\footnote{For $((Z,z),(f,T))$ not in the image of $\kappa$
we set $\rho((Z,z),(f,T))$ to be some arbitrary concept.} of $\kappa$ with respect to $C$.
Let $\eps = \eps(C)$ be as in \eqref{en:whatIsEps}.

\paragraph{Induction base.}
The induction base here applies to the same classes like the induction base
of the compression map.
This is the only case where $T=0$,
and we apply the reconstruction map of Floyd and Warmuth~\cite{DBLP:conf/colt/Floyd89,DBLP:journals/ml/FloydW95} 

\paragraph{Induction step.}
Distinguish between two cases:

\begin{enumerate}[\bf{Case} 1:] 
\item \label{rho:1}
$f$ is defined on $T$. 

Let $x = f(T)$.
Denote
\begin{align*}
& X' = X - \{x,r(x)\},\\
& C' = \{ c'|_{X'}:c'\in C, c'(x)=z(x),c'(r(x))=1-z(x)\} ,\\
& Z' = Z-\{x,r(x)\},\\
& z' = z|_{Z'},\\
& T' = T-1,\\
& f' = f|_{\{0,\ldots,T'\}}.
\end{align*}
Apply recursively $\rho$ on $C', ((Z',z'),(f',T'))$. Let $h'\in\{0,1\}^{X'}$ be the result.
Output $h$ where
\begin{align*}
& h|_{X'}=h',\\
& h(x)=z(x),\\
& h(r(x))=1-z(x).
\end{align*}

\item \label{rho:2}
$f$ is not defined on $T$.

Consider $r(Z)=\{r(x): x\in Z\} \subseteq A^*(C,\eps)$.
For each $x'\in r(Z)$, pick $s(x')\in Z$ to be an element such that $r(s(x'))=x'$. 
Let
\begin{align*}
& X' = A^*(C,\eps) , \\
& C' = C|_{X'},\\
& Z' = r(Z) ,\\
& z'(x') = z(s(x')) \ \forall x'\in Z',\\
& T' = T-1,\\
& f' = f|_{\{0,\ldots,T'\}}.
\end{align*}
Apply recursively $\rho$ on $C', ((Z',z'),(f',T'))$ and let $h'\in\{0,1\}^{X'}$ be the result.
Output $h$ satisfying
\begin{align*}
h(x) = h'(r(x)) \ \forall x\in X.
\end{align*}
\end{enumerate}

\subsection{Correctness}\label{sec:correct}

The following lemma yields the correctness of the compression scheme.

\begin{lemma}
Let $C$ be a concept class,  $(Y,y)\in L_C(1,n)$,
 $\kappa(Y,y) = ((Z,z),(f,T)) $ and  $h=\rho(\kappa(Y,y))$.
Then,
\begin{enumerate}
\item \label{enum:1}
$Z\subseteq Y$ and $z|_Z=y|_Z$, and
\item \label{enum:2}
$h|_Y=y|_Y$.
\end{enumerate}
\end{lemma}
\begin{proof}
We proceed by induction on $n(C)$.
In the base case, $|C| \leq (4e^2)^{d\cdot 2^d + 1}$ and the lemma follows
from the correctness of Floyd and Warmuth's compression scheme
(this is the only case in which $T=0$).
In the induction step, assume $|C| > (4e^2)^{d\cdot 2^d + 1}$.
We distinguish between two cases:

\begin{enumerate}[\bf{Case} 1:] 

\item $f$ is defined on $T$. 

Let $x = f(T)$. 
This case corresponds to Case~\ref{kappa:1} in the definitions of $\kappa$ and  Case~\ref{rho:1} in the definition of $\rho$.
By Item~\ref{prop:defT} of Lemma~\ref{property:compScheme},
$x \in Y$ and there exists $c\in C$ and $x\in Y$ such that $c|_Y=y$ and $c(r(x)) \neq c(x)$.
Let $C', (Y',y')$ be the class defined in Case~\ref{kappa:1} in the definition of $\kappa$.
Since $n(C') < n(C)$, we know that $\kappa,\rho$ on $C'$ satisfy the induction hypothesis.
Let
\begin{align*} 
& ((Z',z'),(f',T')) = \kappa(C',(Y',y')), \\ 
& h' = \rho(C', ((Z',z'),(f',T'))) , 
\end{align*}
be the resulting compression and reconstruction. 
Since we are in Case~\ref{kappa:1} in the definitions of $\kappa$ and  Case~\ref{rho:1} in the definition of $\rho$, $((Z,z),(f,T))$ and $h$ have the following form:
\begin{align*}
& Z=Z'\cup\{x\} , \\ 
& z|_{Z'}=z' ,\ z(x)=y(x) , \\ 
& T=T'+1 , \\
&f|_{\{0,\ldots,T-1\}}=f'|_{\{0,\ldots,T-1\}}, &\\
& f(T)=x,
\end{align*}
and
\begin{align*}
& h|_{X-\{x,r(x)\}}=h',\\
& h(x)=z(x) = y(x) = c(x) ,\\
& h(r(x))=1-z(x) = 1-y(x) = 1-c(x) = c(r(x)) .
\end{align*}

Consider item \ref{enum:1} in the conclusion of the lemma.
By the definition of $Y'$ and $x$,
\begin{align*}
Y'\cup\{x\} &\subseteq Y, \tag{by the definition of $Y'$}\\
Z' & \subseteq Y'. \tag{by the induction hypothesis}
\end{align*}
Therefore, $Z=Z'\cup\{x\}\subseteq Y$.

Consider item \ref{enum:2} in the conclusion of the lemma.
By construction and induction, 
$$h|_{Y\cap \{x,r(x)\}} = c|_{Y\cap \{x,r(x)\}} = y|_{Y \cap \{x,r(x)\}}
 \ \ \text{and} \ \
 h|_{Y'} = h'|_{Y'} = y'.$$
Thus, $h|_Y = y$.

\item $f$ is not defined on $T$. 

This corresponds to Case~\ref{kappa:2} in the 
definitions of $\kappa$ and Case~\ref{rho:2} in the definition of $\rho$.
Let $C', (Y',y')$ be the result of  Case~\ref{kappa:2} in the definition of $\kappa$.
Since $n(C') < n(C)$, we know that $\kappa,\rho$ on $C'$ satisfy the induction hypothesis.
Let 
\begin{align*} 
& ((Z',z'),(f',T')) = \kappa(C',(Y',y')), \\
& h' = \rho(C', ((Z',z'),(f',T'))), \\
& s:Y'\rightarrow Y,
\end{align*}
as defined in Case~\ref{kappa:2} in the 
definitions of $\kappa$ and Case~\ref{rho:2} in the definition of $\rho$.
By construction, $((Z,z),(f,T))$ and $h$ have the following form:
\begin{align*}
& Z=\{s(x'):x'\in Z'\}, \\ 
& z(x) = z'(r(x)) \ \forall x\in Z, \\
& T=T'+1, \\
& f=f',
\end{align*}
and
\begin{align*}
& h(x) = h'(r(x)) \ \forall x\in X.
\end{align*}
\end{enumerate}

Consider item \ref{enum:1} in the conclusion of the lemma.
Let $x\in Z$. By the induction hypothesis, $Z'\subseteq Y'$. Thus, $x=s(x')$ for some $x'\in Z'\subseteq Y'$. Since the range of $s$ is $Y$, it follows that $x\in Y$. This shows that $Z\subseteq Y$.

{
Consider item \ref{enum:2} in the conclusion of the lemma.
For $x\in Y$,
\begin{align*}
h(x) &= h'(r(x))\tag{by the definition of $h$} \\
       &= y'(r(x))\tag{by the induction hypothesis}\\
       & = y(s(r(x))) \tag{by the definition of $y'$ in Case~\ref{kappa:2} of $\kappa$} \\
       &= y(x), 
\end{align*}
where the last equality holds due to item~\ref{prop:notdefT} of Lemma~\ref{property:compScheme}:
Indeed, let $c \in C$ be so that $c|_Y = y$.
Since $f$ is not defined on $T$, 
for all $x \in Y$ we have $c(x)= c(r(x))$.
In addition, for all $x \in Y$ it holds that $r(s(r(x))) = r(x)$ and $s(r(x)) \in Y$.
Hence, if $y(s(r(x))) \neq y(x)$ then one of them is different than $c(r(x))$, contradicting the assumption that we are in Case~\ref{kappa:2} of $\kappa$.}
\end{proof}

\subsection{The compression size}\label{sec:size}

Consider a concept class $C$ which is not part of the induction base (i.e.\ $|C| >(4e^2)^{d\cdot 2^d +1}$).
Let $\eps = \eps(C)$ be as in \eqref{en:whatIsEps}.
We show the effect of each case in the definition of $\kappa$ on either $|C|$ or $n(C)$:
\begin{enumerate}
\item \label{size:1} Case~\ref{kappa:1} in the definition of $\kappa$: Here the size of $C'$ becomes smaller 
$$|C'|\leq\eps|C|.$$ 
Indeed, this holds as in the dual set system $C^*$, the normalized hamming distance between $c_x$ and $c_{r(x)}$ is at most $\eps$ and therefore the number of $c\in C$ such that $c(x)\neq c(r(x))$ is at most $\eps |C|$.
\item \label{size:2} Case~\ref{kappa:2} in the definition of $\kappa$: here $n(C')$ becomes smaller as
$$n(C')=|A^*(C,\eps)|\leq \left(\frac{1}{\eps}\right)^{2^{d}}.$$
\end{enumerate}
We now show that in either cases, $|C'|  \leq  |C|^{1-\frac{1}{d\cdot2^d+1}}$, which implies that after 
$$O((d\cdot 2^d+1)\log\log|C|)$$ iterations,
we reach the induction base.\\
In Case~\ref{size:1}:
\begin{align*}
|C'| & \leq \eps |C|  = |C|^{1-\frac{1}{d\cdot2^d+1}}.\tag{by the definition of $\eps$}
\end{align*}
In Case~\ref{size:2}:
\begin{align*}
|C'| & \leq (n(C'))^d \tag{by Theorem~\ref{thm:Sauer}, since $VC(C') \leq d-2$} \\
      & \leq \left(\frac{1}{\eps}\right)^{d\cdot 2^d}\tag{by Theorem~\ref{thm:Haussler}, since $n(C')=|A^*(C,\eps)|$}\\
      & = |C|^{1-\frac{1}{d\cdot2^d+1}} . \tag{by definition of $\eps$}
\end{align*}

\begin{remark}
Note the similarity between the analysis of the cases above,
and the analysis of the size of a teaching set in Section~\ref{sec:RTD}. Case~\ref{size:1}
corresponds to the rate of the progress performed in each iteration of the construction of a teaching set. Case~\ref{size:2} corresponds to the calculation showing that in each iteration significant progress can be made.
\end{remark}

Thus, the compression map $\kappa$ performs at most
$$O((d\cdot 2^d+1)\log\log|C|)$$ iterations. In every step of the recursion the sizes of $Z$ and $T$ increase by at most $1$. In the base of the recursion, $T$ is $0$ and the size of $Z$ is at most $O(d \cdot 2^d)$. 
Hence, the total size of the compression satisfies
\begin{align*}
& |Z| \leq k = O(2^d d\log\log |C|),\\
& \log(|Q|) \leq O(k \log(k)).
\end{align*}
This completes the proof of Theorem~\ref{thm:CompSch}.

\section*{Acknowledgements}

We thank Noga Alon and Gillat Kol
for helpful discussions in various stages of this work.

\bibliographystyle{plain} 

\bibliography{compRef}

\appendix

\section{Double sampling}
\label{sec:DoubleSampling}

Here we provide our version of the double sampling
argument from \cite{zbMATH04143473} that upper bounds
the sample complexity of PAC learning
for classes of constant  VC-dimension.
We use the following simple general lemma.

\begin{lemma}\label{lem:aux}
Let $(\Omega,\mathcal{F},\mu)$ and
$(\Omega',\mathcal{F}',\mu')$ be countable\footnote{A similar
statement holds in general.} probability spaces.
Let 
$$F_1,F_2,F_3,\ldots  \in \mathcal{F} , \ 
F'_1,F'_2,F'_3,\ldots \in \mathcal{F}'$$ be so that
$\mu'(F'_{i})\geq 1/2$ for all $i$.
Then
$$\mu \times \mu' \left(\bigcup_i {F_i\times F'_i} \right)\geq
\frac{1}{2}\mu \left( \bigcup_i {F_i} \right),$$
where $\mu\times \mu'$ is the product measure.
\end{lemma}

\begin{proof}
Let $F = \bigcup_i F_i$.
For every $\omega \in F$, 
let $F'(\omega)= \bigcup_{i: \omega \in F_i} F'_i$.
As there exists $i$ such that $\omega\in F_i$ it holds that $F'_i\subseteq F'(\omega)$ and hence $\mu'(F'(\omega)) \geq 1/2$.
Thus,
\begin{align*} 
\mu \times \mu' \left(\bigcup_i {F_i\times F'_i} \right) 
= \sum_{\omega \in F} \mu(\{\omega\}) \cdot \mu'(F'(\omega)) 
\geq \sum_{\omega \in F} \mu(\{\omega\})/2 = \mu(F)/2.
\end{align*}
\end{proof}

We now give a proof of  Theorem~\ref{thm:PAC}. To ease the reading we repeat the statement of the theorem.

\begin{nonumtheorem}
Let $X$ be a set and $C \subseteq \{0,1\}^X$ be a concept class of VC-dimension $d$.
Let $\mu$ be a distribution over $X$.
Let $\eps,\delta >0$ and $m$ an integer satisfying $2  (2m+1)^{d} (1-\eps/4)^{m}  < \delta$.
Let $c\in C$ and  $Y = (x_1,\dots,x_m)$ 
be a multiset of $m$ independent samples from $\mu$.
Then, the probability that there is $c' \in C$
so that $c|_Y = c'|_Y$ but
$\mu(\{x : c(x) \neq c'(x)\}) > \eps$ is at most $\delta$.
\end{nonumtheorem}

\begin{proof}[Proof of Theorem~\ref{thm:PAC}]
Let $Y' = (x'_1,\ldots,x'_m)$ be another $m$ independent samples from $\mu$,
chosen independently of $Y$.  
Let
$$H = \{h \in C : \dist_\mu(h,c) > \eps  \}.$$
For $h \in C$, define the event
$$F_h =  \{ Y : c|_Y = h|_Y  \},$$
and let
$F = \bigcup_{h \in H} F_h$.
Our goal is thus to upper bound $\Pr(F)$.
For that, we also define the independent event
$$F'_h = \{ Y' : \dist_{Y'}(h,c) >  \eps/2 \}.$$
We first claim that $\Pr(F'_h) \geq 1/2$ for all $h \in H$.
This follows from Chernoff's bound,
but even Chebyshev's inequality suffices:
For every $i \in [m]$, let $V_i$ be the indicator 
variables of the event $h(x'_i) \neq c(x'_i)$ (i.e., $V_i=1$ if and only if  $h(x'_i) \neq c(x'_i)$).
The event $F'_h$ is equivalent to $V = \sum_i V_i/m > \eps /2$.
Since $h \in H$, we have $p:= \mathbb{E} [V] > \eps$.
Since elements of $Y'$ are chosen independently, it follows that $\Var(V) = p (1-p)/m$.
Thus, the probability of the complement of $F'_h$ satisfies
\begin{align*}
\Pr((F'_h)^c) \leq \Pr(|V - p| \geq p-\eps /2)
\leq \frac{p(1-p)}{(p-\eps/2)^2 m} 
< \frac{4 }{\eps  m} \leq 1/2.
\end{align*}
We now give an upper bound on $\Pr(F)$. We note that
\begin{align*}
\Pr(F) 
&\leq 2\Pr \left( \bigcup_{h \in H} F_h \times F'_h \right). \tag{Lemma~\ref{lem:aux}}
\end{align*}
Let $S = Y \cup Y'$, where the union is as multisets.
Conditioned on the value of $S$, the multiset $Y$ is a uniform
subset of half of the elements of $S$.
Thus,
\begin{align*}
2\Pr \left( \bigcup_{h \in H} F_h \times F'_h \right) &=  2 \Ex_{S} \big[
\Ex \big[ \mathds{1}_{\{  \exists h \in H :  h|_Y = c|_Y, \ \dist_{Y'}(h,c) >  \eps/2  \}} \big| S \big] \big] \\
&=  2 \Ex_{S} \big[
\Ex \big[ \mathds{1}_{\{  \exists h' \in H|_S : h'|_Y = c|_Y, \ \dist_{Y'}(h',c) >  \eps/2  \}} \big| S \big] \big] \\
& \leq  2 \Ex_{S} \left[ \sum_{h' \in H|_{S}}
\Ex \big[ \mathds{1}_{\{   h'|_Y = c|_Y, \ \dist_{Y'}(h',c) >  \eps/2  \}} \big| S \big] \right] . \tag{by the union bound} 
\end{align*}
Notice that if $\dist_{Y'}(h',c) > \eps/2$ then $\dist_S(h',c) > \eps/4$, hence the probability that we choose $Y$ such that 
$h'|_Y = c|_Y$ is at most $(1-\eps/4)^{m}$.
Using Theorem~\ref{thm:Sauer} we get
\begin{align*}
\Pr(F) 
& \leq  2 \Ex_{S} \left[ \sum_{h' \in H|_{S}} (1-\eps/4)^{m} \right] 
\leq 2  (2m+1)^{d} (1-\eps/4)^{m}   . 
\end{align*}
\end{proof}

\end{document}